\documentclass[review]{elsarticle}
\usepackage{hyperref}
\usepackage{xcolor}
\usepackage{algorithm}
\usepackage{algorithmic}
\usepackage{amsmath}
\usepackage{amssymb}
\usepackage{multirow} 
\usepackage{subfig}
\usepackage{amsthm}
\usepackage{graphicx}
\DeclareMathOperator*{\argmax}{\arg\!\max}
\usepackage{url}
\renewcommand{\vec}[1]{\mathbf{#1}}
\usepackage[utf8x]{inputenc} 
\newtheorem{theorem}{Theorem}

\begin{document}
\begin{frontmatter}
\title{ Structured Graph Learning for Clustering and Semi-supervised Classification}
\author[label1]{Zhao Kang}
\author[label2]{Chong~Peng}
\author[label3]{Qiang~Cheng}
\author[label4]{ Xinwang Liu}
\author[label5]{Xi Peng}
\author[label6]{Zenglin Xu}
\author[label1]{Ling Tian}
\address[label1]{School of Computer Science and Engineering,  University of Electronic Science and Technology of China, Chengdu, 611731, China.}
\address[label2]{College of Computer Science and Technology, Qingdao University, Qingdao, 266071, China.}
\address[label3]{ Institute of Biomedical Informatics and Department of Computer Science, University of Kentucky, Lexington, KY, 40506, USA.}
\address[label4]{School of Computer Science, National University of Defense Technology, Changsha, 410073, China.}
\address[label5]{ College of Computer Science, Sichuan Univerisity, Chengdu, 610064, China.}
\address[label6]{ Department of Computer Science and Technology, Harbin Institute of Technology, Shenzhen, 518055, China.}
%

\begin{abstract}
Graphs have become increasingly popular in modeling structures and interactions in a wide variety of problems during the last decade. Graph-based clustering and semi-supervised classification techniques have shown impressive performance. This paper proposes a graph learning framework to preserve both the local and global structure of data. Specifically, our method uses the self-expressiveness of samples to capture the global structure and adaptive neighbor approach to respect the local structure. Furthermore, most existing graph-based methods conduct clustering and semi-supervised classification on the graph learned from the original data matrix, which doesn't have explicit cluster structure, thus they might not achieve the optimal performance. By considering rank constraint, the achieved graph will have exactly $c$ connected components if there are $c$ clusters or classes. As a byproduct of this, graph learning and label inference are jointly and iteratively implemented in a principled way. Theoretically, we show that our model is equivalent to a combination of kernel k-means and k-means methods under certain condition. Extensive experiments on clustering and semi-supervised classification demonstrate that the proposed method outperforms other state-of-the-art methods.
\end{abstract}

\end{frontmatter}


\section{Introduction}\label{sec:introduction}

As a natural way to represent structure or connections in data, graphs have broad applications including world wide web, social networks, information retrieval, bioinformatics, computer vision, natural language processing, and many others. Some special cases of graph algorithms, such as graph-based clustering \cite{li2015robust,huang2019auto},  graph embedding \cite{yan2007graph}, graph-based semi-supervised classification \cite{zhang2013graph}, signal processing \cite{shuman2013emerging}, have attracted increasing attention in the recent years.

Clustering refers to the task of finding subsets of similar samples and grouping them together, such that samples in the same cluster would share high similarity to each other, whereas samples in different groups are dissimilar \cite{shen2017compressed,huang2020auto}.
By leveraging a small set of labeled data, semi-supervised classification aims at determining the labels of a large collection of unlabeled samples based on relationships among the samples \cite{li2015learning}. In essence, both clustering and semi-supervised classification algorithms are trying to predict labels for samples \cite{kang2020robust}. As fundamental techniques in machine learning and pattern recognition, they have been facilitating various research fields and have been extensively studied. 

Among numerous clustering and semi-supervised classification methods developed in the past decades, graph based techniques often provide impressive performance. In general, these methods consist of two key steps. First, an affinity graph is constructed from all data points to represent the similarity among the samples. Second, spectral clustering \cite{ng2002spectral} algorithm or label propagation \cite{zhang2017robust} method is utilized to obtain the final labels. Therefore, the start step of building graph might heavily impact the subsequent step and finally lead to suboptimal performance. Since underlying structures of data are often unknown in advance, this pose a major challenge for graph construction. Consequently, the final result might be far from optimal. Unfortunately, constructing a good graph that best captures the essential data structure is still known to be fundamentally challenging \cite{zhu2019spectral}.

The existing strategies to define adjacency graph can be roughly divided into three categories: a) the metric based approaches, which use some functions to measure the similarity among data points \cite{zhu2003semi}, such as Cosine, Euclidean distance, Gaussian function; b) the local structure approaches, which induce the similarity by representing each datum as a linear combination of local neighbors \cite{wang2009clustering} or learning a probability value for two points as neighbors  \cite{nie2017multi}; c) the global self-expressiveness property based approaches, which encode each datum as a weighted combination of all other samples, i.e., its direct neighbors and reachable indirect neighbors \cite{kang2017twin,zhuang2012non}. The traditional metric based approaches and the local neighbor based methods depend upon the selection of metric or the local neighborhood parameter, which heavily influence final accuracy. Hence, they are not reliable in practice \cite{de2013influence}.

On the other hand, adaptive neighbor \cite{nie2017multi} and self-expressiveness approaches \cite{peng2020deep,liu2013robust} automatically learn graph from data. As a matter of fact, they share a similar spirit as locality preserve  projection (LPP) and locally linear embedding (LLE), respectively. Different from LPP and LLE, they don't specify the neighborhood size and predefine the similarity graph. In realistic applications, they enjoy several benefits. First, automatically determining the most informative neighbors for each data point will avoid the inconsistent drawback in widely used $k$-nearest-neighborhood and $\epsilon$-nearest-neighborhood graph construction techniques, which provide unstable performance with respect to different $k$ or $\epsilon$ values \cite{maier2009influence}. Second, they are independent of measure metric, while traditional methods are often data-dependent and sensitive to noise and outliers \cite{kang2017clustering}. Third, they can tackle data with structures at different scales of size and density \cite{kang2019partition}. Therefore, they are prefered in practice.  For example, \cite{zhang2010graph} performs dimension reduction and graph learning based on adaptive neighbor in a unified framework.

Nevertheless, they emphasize different aspects of data structure information, i.e., local and global, respectively. As demonstrated in many problems, such as dimension reduction \cite{hou2009stable}, feature selection \cite{zhu2017robust}, semi-supervised classification \cite{zhou2004learning}, clustering \cite{wang2009clustering}, local and global structure information are both important to algorithm performance since they can provide complementary information to each other and thus enhance the performance. In the paper, we combine them into a unified framework for graph learning task.

Moreover, most existing graph-based methods conduct clustering and semi-supervised classification on the graph learned from the original data matrix, which doesn't have explicit cluster structure, thus they might not achieve the optimal performance. For example, the seminal work \cite{liu2013robust} assumes a low-rank structure of graph, whose solution might not be optimal due to the bias of nuclear norm \cite{kang2015pca}. Ideally, the achieved graph should have exactly $c$ connected components if there are $c$ clusters or classes. Most existing methods fail to take this information into account. In this paper, we consider rank constraint to meet this requirement. As an extension to our previous work \cite{kang2017clustering}, we establish the theoretical connection of our clustering model to kernel k-means and k-means and consider semi-supervised classification application. As an added bonus, graph learning and label inference are seamlessly integrated into a unified objective function. This is quite different from traditional ways, where graph learning and label inference are performed in two separate steps, which easily lead to suboptimal results. To overcome the limitation of single kernel method, we further extend our model to accommodate multiple kernels.  

Though there are many other lines of research on graph. For instance, \cite{kondor2008skew}discusses the transformation issue; \cite{daitch2009fitting} introduces a fitness metric to learn the adjacency matrix; \cite{eldridge2016graphons} focuses on the graph that is sampled from a graphon. Different from them, this work aims to learn a graph that has explicit cluster structure. In particular, the number of clusters/classes is employed as a prior knowledge to enhance the quality of graph, which leads to improved performance of clustering and semi-supervised classification. Additionally, graph neural networks (GNN) has gained increasing popularity recently \cite{kipf2016semi}. The main difference between GNN and our method is that GNN targets to process a graph that is already available in existing data, while our method is designed to learn a good graph from feature data for further processing. Hence, our method and GNN focus on different types of data. In practice, feature data is more common than graph data. From this point of view, our method could be useful for GNN applications when the graph is not available or the graph has low quality. As a matter of fact, how to refine the graph used in GNN is a promissing research direction.

To sum up, the main contributions of this paper are: 
\begin{enumerate}
\item The similarity graph and labels are adaptively learned from the data by preserving both global and local structure information. By leveraging the interactions among them, they are mutually reinforced towards an overall optimal solution.
\item Theoretical analysis shows the connections of our model to kernel k-means, k-means, and spectral clustering methods. Our framework is more general than k-means and kernel k-means. At the same time, it solves the graph construction challenge of spectral clustering.
\item Based on our method with a single kernel, we further extend our model into an integrated framework which can simultaneously learn the similarity graph, labels, and the optimal combination of multiple kernels. Each subtask can be iteratively boosted by using the results of the others.
\item Extensive experiments on real-world data sets are conducted to testify the effectiveness and advantages of our framework over other state-of-the-art clustering and semi-supervised classification algorithms.
\end{enumerate}
The rest of the paper is organized as follows. Section \ref{single} introduces the proposed clustering method based on a single kernel. In Section \ref{theoretical}, we show the theoretical analysis of our model. An extended model with multiple kernel learning ability is provided in Section \ref{multiple}. Clustering and semi-supervised classification experimental results and analysis are presented in Section \ref{clusterexperiment} and \ref{semiexperiment}, respectively. Section \ref{conclusion} draws conclusions.

\textbf{Notations.} Given a data set $X\in\mathcal{R}^{n\times m}$ with $m$ features and $n$ instances, its $i$-th sample and
$(i,j)$-th element are denoted by $x_i\in\mathcal{R}^{m\times 1}$ and $x_{ij}$, respectively. The $\ell_2$-norm of $x_i$ is denoted as $\|x_i\|= \sqrt{x_i^T \cdot x_i}$, where $T$ means transpose. The definition of squared Frobenius norm is $\|X\|_F^2=\sum\limits_{ij}x_{ij}^2$. $I$ represents the identity matrix and $\vec{1}$ denotes a column vector with all the elements as one. $Tr(\dot)$ is the trace operator. $0\leq Z\leq 1$ indicates that elements of $Z$ are in the range of $[0,1]$.


\section{Structured Graph Learning with Single Kernel }
\label{single}
In this section, we first review local and global structure learning, then describe our model and its optimization.
\subsection{Local Structure Learning}
It is reasonable to assume that the similarity $z_{ij}$ between the $i$-th sample $x_i$ and the $j$-th sample $x_j$ is big if distance $\|x_i-x_j\|^2$ is small. Intuitively, we follow the addaptive neighbor approach \cite{nie2017multi} to have the following objective function:
\begin{equation}
\min_{z_i} \sum_{j=1}^n (\| x_i-x_j\|^2z_{ij}+\alpha z_{ij}^2) \quad s.t. \quad z_i^T\vec{1}=1, \quad  0\leq z_{ij}\leq 1,
\label{local}
\end{equation}
where $\alpha$ is a tuning parameter and it can be computed in advance as we show later. By solving above problem, we obtain a graph matrix $Z\in \mathcal{R}^{n\times n}$, which characterizes the pairwise relationships among samples.  

Define $d_{ij}^x=\|x_i-x_j\|^2=x_i^Tx_i+x_j^Tx_j-2x_i^Tx_j$, then its corresponding matrix is $D^x=Diag(XX^T)\vec{1}\vec{1}^T+\vec{1}\vec{1}^TDiag(XX^T)-2XX^T$, where $Diag(XX^T)$ is a diagonal matrix with the diagonal elements of $XX^T$. Thus (\ref{local}) can be reformulated in matrix format as:
\begin{equation}
\min_{Z} Tr(Z^TD^x)+\alpha \|Z\|_F^2 \quad s.t. \quad Z^T\vec{1}=\vec{1}, \quad  0\leq Z\leq 1.
\label{locals}
\end{equation}
The achieved graph $Z$ from (\ref{locals}) will capture the local structure information. Since choosing local neighbors may lead to disjoint components and incorrect neighbors, we advocate to preserve global neighborhoods.
\subsection{Global Structure Learning}
Self-expressive property has been applied to many applications and demonstrates its capability in capturing the global structure of data \cite{tang2019feature,zhan2018multiview}. In particular, subspce clustering is built on this property to learn an adjacency matrix \cite{zhang2020latent,kang2019partition, peng2020deep}. It assumes that each data point can be linearly reconstructed from weighted combinations of all other data points, i.e., its direct neighbors and reachable indirect neighbors. The weight coefficient matrix $Z$ also behaves like similarity matrix, since the weight $z_{ij}$ should be big if $x_i$ and $x_j$ are similar. In mathematical language, this problem is written as:
\begin{equation}
\min_{Z} \|X^T-X^TZ\|_F^2+\alpha f(Z) \hspace{.1cm}  s.t. \hspace{.1cm} Z^T\vec{1}=\vec{1}, \hspace{.1cm} 0\leq Z\leq 1.
\label{global}
\end{equation}
where $f(Z)$ is a regularizer on $Z$. For simplicity, squared Frobenius norm of $Z$ is adopted in this paper. As a result, (\ref{global}) will learn the graph matrix by following the distribution of the data points, which will reflect the global relationships. 

It is easy to see that (\ref{global}) is a linear model and assumes that data points are drawn from a union of subspaces. Hence, it may not work well when data points reside in a union of manifolds. As we know, nonlinear data display linearity if mapped to an implicit, higher-dimensional space \cite{liu2018late,kang2020relation}. Therefore, we extend (\ref{global}) to kernel representation through transformation $\phi$, then $K_{ij}=<\phi(x_i),\phi(x_j)>$ . It yields:
\begin{equation}
\begin{split}
\min_{Z}  \hspace{.1cm} &Tr(K-2KZ+Z^TKZ)+\alpha f(Z) \hspace{.1cm} \\
&s.t.  \hspace{.1cm} Z^T\vec{1}=\vec{1},\hspace{.1cm} 0\leq Z\leq 1.
\label{globalkk}
\end{split}
\end{equation}
(\ref{globalkk}) will recover the nonlinear relationships in the raw space. 

To make use of possible complementary information provided by the local structure and the global structure of the samples, we combine (\ref{locals}) and (\ref{globalkk}) into a single unified objective function: 
\begin{equation}
\begin{split}
\min_{Z}  \hspace{.1cm}  &Tr(K-2KZ+Z^TKZ)+ Tr(Z^TD^x)+\alpha \|Z\|_F^2\\
&s.t.  \hspace{.1cm} Z^T\vec{1}=\vec{1},  \hspace{.1cm} 0\leq Z\leq 1.
\end{split}
\label{adaptive}
\end{equation}
As a consequence, (\ref{adaptive}) will provide a graph matrix $Z$ that respects the global and local structure hidden in the data. However, $Z$ doesn't display an explicit cluster structure, thus it may not produce the optimal performance. Specifically, we expect that the connections among data samples from different classes are as weak as possible; whereas the connections among data points within the same class are as strong as possible. Ideally, the achieved graph should have exactly $c$ connected components if therer are $c$ clusters or classes, i.e., $Z$ is block diagonal (with proper permutations) in which each block is connected and corresponds to data samples from the same class. 
\subsection{Structured Graph Learning}
To achieve the desired structure of graph matrix $Z$, we impose constraint on the rank of its Laplacian, which is defined as $L=D-\frac{Z+Z^T}{2}$, where $D\in\mathcal{R}^{n\times n}$ is the diagonal degree matrix with $d_{ii}=\sum_j \frac{z_{ij}+z_{ji}}{2}$. Concretely, we are based on the following important theorem \cite{kang2017twin}. 
\begin{theorem}
\label{threm1}
The multiplicity $c$ of the eigenvalue 0 of the Laplacian matrix $L$ is equal to the number of connected components in the graph associated with $Z$.
\end{theorem}
Theorem \ref{threm1} indicates that $rank(L)=n-c$ if $Z$ contains exactly $c$ connected components. Thus our proposed \textbf{S}tructured \textbf{G}raph learning framework with \textbf{S}ingle \textbf{K}ernel (SGSK) is:
\begin{equation}
\begin{split}
&\min_Z Tr(K-2KZ+Z^TKZ)+ Tr(Z^TD^x)+\alpha \|Z\|_F^2  \\
&s.t.\quad  Z^T\vec{1}=\vec{1}, \quad 0\leq Z\leq 1,\quad rank(L)=n-c.
\end{split}
\label{sgsk}
\end{equation}
The problem (\ref{sgsk}) seems very difficult to solve since $L$ also depends on $Z$. In the next subsection, we will design a novel algorithm to solve this problem. 
\subsection{Optimization}
Let $\sigma_i(L)$ denotes the $i$-th smallest eigenvalue of $L$. Since $L$ is positive semi-definite, we have $\sigma_i(L)\geq 0$. Then $rank(L)=n-c$ means $\sum_{i=1}^{c} \sigma_i(L)=0$. The problem (\ref{sgsk}) is equivalent to the following problem for a large enough $\gamma$: 
\begin{equation}
\begin{split}
\min_Z\hspace{.1cm}  &Tr(K\!-\!2KZ\!+\!Z^TKZ)\!+\!Tr(Z^TD^x)\!+\!\alpha \|Z\|_F^2\\ &+\!\gamma\! \sum_{i=1}^{c}\! \sigma_i(L)\quad s.t.\quad  Z^T\vec{1}=\vec{1}, \quad 0\leq Z\leq 1.
\end{split}
\label{new}
\end{equation}
According to the Ky Fan's Theorem \cite{kang2017twin}, we have:
\begin{equation}
\sum_{i=1}^{c} \sigma_i(L)=\min_{P^TP=I} Tr(P^TLP),
\end{equation}
where $P\in\mathcal{R}^{n\times c}$ is the cluster/label matrix. Therefore, the problem (\ref{new}) can be reformulated as:
\begin{equation}
\begin{split}
\min_{Z,P}&\hspace{.1cm} Tr(K-2KZ+Z^TKZ)+Tr(Z^TD^x)+\alpha\|Z\|_F^2+\\
&\gamma Tr(P^TLP) \quad s.t.\hspace{.1cm}  Z^T\vec{1}=\vec{1}, \hspace{.1cm} 0\leq Z\leq 1,\hspace{.1cm} P^TP=I.
\end{split}
\label{newmodel}
\end{equation}
Then we can solve problem (\ref{newmodel}) using an alternating optimization strategy. 

When $Z$ is fixed, the problem (\ref{newmodel}) becomes:
\begin{equation}
\min_{P^TP=I} Tr(P^TLP).
\label{solvep}
\end{equation}
The optimal solution $P$ is formed by the $c$ eigenvectors of $L$ corresponding to the $c$ smallest eigenvalues.

When $P$ is fixed, the problem (\ref{newmodel}) becomes:
\begin{equation}
\begin{split}
\min_{Z}&\hspace{.1cm} Tr(K-2KZ+Z^TKZ)+Tr(Z^TD^x)+\alpha\|Z\|_F^2+\\
&\gamma Tr(P^TLP) \quad s.t.\hspace{.1cm}  Z^T\vec{1}=\vec{1}, \hspace{.1cm} 0\leq Z\leq 1.
\end{split}
\label{model}
\end{equation}
According to the property of Laplacian matrix, we have the following equation:
\begin{equation}
\sum_{i,j}\frac{1}{2}\|P_{i,:}-P_{j,:}\|^2z_{ij}=Tr(P^TLP)
\end{equation}
Based on it, the problem (\ref{model}) can be rewritten in the vector form as:
\begin{equation}
\begin{split}
\min_{z_i} \hspace{.1cm}&z_i^T(\alpha I+K)z_i+[( d_i^x+\frac{\gamma}{2}d_i^p)^T-2K_{i,:}]z_i\\
&s.t.\quad z_i^T\vec{1}=1,\quad 0\leq z_{ij}\leq 1.
\end{split}
\label{solveZ}
\end{equation}
where we denote $d_i^p\in\mathcal{R}^{n\times 1}$ as a vector with the $j$-th element $d_{ij}^p=\|P_{i,:}-P_{j,:}\|^2$. Note that the nearest neighbors to any data point $x_i$ are not steady and they change in each iteration. Thus the neighbors are learned adaptively here, which is quite different from traditional approaches. Problem (\ref{solveZ}) can be solved in parallel by various quadratic programing packages. 

We can observe that when graph $Z$ is given, our algorithm solves a spectral clustering problem; when $P$ is known, our algorithm learns graph to well respect the local and global strucure of the data under the guidance of the cluster structure. For clarity, the complete procedure is outlined in Algorithm 1.
\begin{algorithm}[!tb]
\caption{The algorithm of SGSK }
\label{alg1}
 {\bfseries Input:} Kernel matrix $K$, parameter $\gamma>0$, $\alpha$.\\
{\bfseries Initialize:} Random matrix $Z$.\\
 {\bfseries REPEAT}
\begin{algorithmic}[1]
 \STATE Calculate $P$ as the $c$ smallest eigenvectors of $L=D-\frac{Z+Z^T}{2}$.
   \STATE For each $i$, update the $i$-th column of $Z$ according to (\ref{solveZ}).
\end{algorithmic}
\textbf{ UNTIL} {stopping criterion is met.}
\end{algorithm}
\subsection{Convergence Analysis}
SGSK is solved in an alternative way, the optimization procedure will monotonically decrease the objective function value of the problem in (\ref{newmodel}) in each iteration \cite{bezdek2003convergence}. Since the objective function has a lower bound, such as zero, the above iteration converges. 
\subsection{Determination of Parameter $\alpha$}
\label{choosealpha}
In our proposed model, parameter $\alpha$ controls the balance between the trivial solution ($\alpha=0$) and the uniform distribution ($\alpha=\infty$). To alleviate computational burden, a sparse $z_i$, i.e., only $x_i$' $k$ nearest neighbors are connected to $x_i$, is expected for local structure learning. Motivated  by this, we introduce a practical way to set $\alpha$ value. 

For subproblem (\ref{local}), its corresponding Lagrangian function is  
\begin{equation}
(d_i^x)^Tz_i+\alpha_i z_i^Tz_i-\beta(z_i^T\vec{1}-1)-\rho_i^Tz_i,
\end{equation}
where $\beta$ and $\rho_i$ are the Lagrangian multipliers. For each $i$, we introduce a parameter $\alpha_i$. By Karush-Kuhn-Tucker (KKT) condition, we have 
\begin{equation}
z_{ij}=(\frac{\beta-d_{ij}^x}{2\alpha_i})_+
\label{sol}
\end{equation}
Considering the constraint $z_i^T\vec{1}=1$, we have
\begin{equation}
\sum_{j=1}^k (\frac{\beta-d_{ij}^x}{2\alpha_i})=1
\Rightarrow \beta=\frac{2\alpha_i+\sum\limits_{j=1}^kd_{ij}^x}{k}
\end{equation}
To keep $k$ nonzero components, we can have $z_{ik}>0$ and $z_{i,k+1}=0$ if we sort each row of $D^x$ in ascending order denoted by $d_{i1}^x, d_{i2}^x,\cdots,d_{in}^x$. Then the following inequalities hold
\begin{equation}
    \begin{cases}
    \frac{\beta-d_{ik'}^x}{2\alpha_i}>0\quad for\quad k'=1,\cdots,k 
\vspace{.2cm}\\
    \frac{\beta-d_{i,{k''}}^x}{2\alpha_i}\leq 0 \quad for \quad k''=k+1,\cdots,n.
    \end{cases}
\label{ine}
  \end{equation}

Inserting the $\beta$ value, we have the following inequality for $\alpha_i$:
\begin{equation}
\frac{k}{2}d_{ik}^x-\frac{1}{2}\sum\limits_{j=1}^kd_{ij}^x<\alpha_i\leq\frac{k}{2}d_{i,k+1}^x-\frac{1}{2}\sum\limits_{j=1}^kd_{ij}^x
\label{eachalpha}
\end{equation}
This range of $\alpha_i$ values will make sure $z_i$ has exactly $k$ nonzero elements. For convenience, we set $\alpha_i=\frac{k}{2}d_{i,k+1}^x-\frac{1}{2}\sum\limits_{j=1}^kd_{ij}^x$. Then, the average number of nonzero elements in each row of $Z$ is close to $k$ if we set $\alpha$ to be the mean value of $\alpha_i, \alpha_2,\cdots, \alpha_n$. That is,
\begin{equation}
\alpha=\frac{1}{n}\sum_{i=1}^n(\frac{k}{2}d_{i,k+1}^x-\frac{1}{2}\sum\limits_{j=1}^kd_{ij}^x).
\end{equation}
In this way, we can avoid tuning $\alpha$ blindly and instead we search the neighborhood size $k\in(0,n]$. 

\section{Theoretical Connection}
\label{theoretical}
\subsection{Connection to Kernel K-means and K-means Clustering}
\begin{theorem}
When $\alpha\to\infty$, the proposed SGSK model is equivalent to a combination of kernel k-means and k-means problems.
\end{theorem}
\begin{proof}
As aforementioned, the constraint $rank(L)=n-c$ in (\ref{sgsk}) will make $Z$ block diagonal. Suppose $Z_i\in\mathcal{R}^{n_i\times n_i}$ is the similarity graph matrix of the $i$-th component, where $n_i$ is the number of data samples in this component. Then problem (\ref{sgsk}) can be written for each $i$:
\begin{equation}
\begin{split}
&\min_{Z_i}  \|\phi(X_i)-\phi(X_i)Z_i\|_F^2+Tr(Z_i^TD_i^x)+\alpha \|Z_i\|_F^2  \\
&s.t.\quad  Z_i^T\vec{1}=\vec{1}, \quad 0\leq Z_i\leq 1,
\end{split}
\label{equal}
\end{equation}
where $X_i$ consists of the points in the $Z_i$. When $\alpha\to\infty$, the above problem becomes:
\begin{equation}
\min_{Z_i} \|Z_i\|_F^2  \quad
s.t.\hspace{.2cm}  Z_i^T\vec{1}=\vec{1}, \hspace{.1cm} 0\leq Z_i\leq 1.
\end{equation}
The solution is all elements in $Z_i$ are with the same value $\frac{1}{n_i}$. 

Therefore, when $\alpha\to\infty$, the solution to problem (\ref{sgsk}) is: 
\begin{equation}
    z_{ij}=
    \begin{cases}
      \frac{1}{n_k}, & \text{if $x_i$ and $x_j$ are in the same $k$-th component} \\
      0, & \text{otherwise}
    \end{cases}
  \end{equation}
Denote the solution set of this form as $\mathcal{C}$. We can see that $\|Z\|_F^2=c$ and $Z\vec{1}=\vec{1}^TZ=\vec{1}$. Thus (\ref{sgsk}) can be written as:
\begin{equation}
\min_{Z\in \mathcal{C}} \sum_i\|\phi(x_i)-\phi(X)z_i\|^2+Tr(Z^TD^x)
\label{kernelf}
\end{equation}
For the first term, it is easy to deduce that $\phi(X)z_i$ is the mean of cluster $c_i$ in the kernel space. Therefore, the first term in (\ref{kernelf}) is exactly the kernel k-means. 

For the second term in (\ref{kernelf}), we first introduce the centering matrix, i.e., $H=I-\frac{1}{n}\vec{1}\vec{1}^T$. It is obvious that $H\vec{1}=\vec{0}$ and also ${\bf{1}}^T H = 0$. It can be shown that $HD^xH=-2HXX^TH$. Moreover, $Tr(Z^TD^x)=Tr(D^xZ)=Tr(HD^xHZ)+\frac{1}{n}\vec{1}^TD^x\vec{1}$. Therefore, we have
\begin{equation}
\begin{split}
&\min_{Z\in \mathcal{C}} Tr(Z^TD^x)
\Longleftrightarrow\min_{Z\in \mathcal{C}} Tr(HD^xHZ)\\
\Longleftrightarrow&\max_{Z\in \mathcal{C}} Tr(HXX^THZ)
\Longleftrightarrow\max_{Z\in \mathcal{C}} Tr(X^THZHX)\\
\Longleftrightarrow&\min_{Z\in \mathcal{C}} Tr(X^TH(I-Z)HX)
\Longleftrightarrow\min_{Z\in \mathcal{C}} Tr(S_w)\\
\label{globalk}
\end{split}
\end{equation}
which is exactly the problem of k-means. Here, $S_w$ is the so-called within-class scatter matrix. 

Therefore, our proposed model is to solve a combination of kernel k-means and k-means clustering problems when $\alpha\to\infty$. When $\alpha$ is not very large, our model becomes a generalization of kernel k-means and k-means, so it can partition data of an arbitrary shape. 
\end{proof}
\subsection{Connection to Spectral Clustering}
With a prespecified graph $Z$, spectral clustering solves the following problem:
\begin{equation}
\min_{P^TP=I} Tr(P^TLP).
\end{equation}
In general, $Z$ does not have exactly $c$ connected components and $P$ may not be optimal. Unlike existing spectral clustering method, $Z$ is not predefined in (\ref{solvep}). Also, $Z$ is achieved by incorporating cluster/class structure. $Z$ and $P$ are learned simultaneously in a coupled way, so that collaboratively improve each of them. This results in overall optimal solutions, which are confirmed by our experiments.

\section{Structured Graph Learning with Multiple Kernel }
\label{multiple}
The only input for our proposed model (\ref{newmodel}) is kernel $K$. It is well known that the performance of kernel method is strongly dependent on the selection of kernel. It is also time consuming and impractical to exhaustively search the optimal kernel. Multiple kernel learning \cite{liu2019multiple} which lets an algorithm do the picking or combination from a set of candidate kernels is an effective way to tackle this issue. Here we present an approach to identify a suitable kernel or construct a consensus kernel from a pool of predefined kernels.

Transforming and concatenating $r$ kernel spaces with different weights $\sqrt{w_i}(w_i\ge 0)$, we have $\tilde{\phi}(x)=[\sqrt{w_1}\phi_1(x),$ $\sqrt{w_2}\phi_2(x),...,\sqrt{w_r}\phi_r(x)]^T$. Then the combined kernel $K_w$ becomes
\begin{equation}
\label{ukernel}
K_w(x,y)=<\tilde{\phi}_w(x),\tilde{\phi}_w(y)>=\sum\limits_{i=1}^r w_iK^i(x,y).
\end{equation} 
Replacing single kernel with combined kernel, we obtain our proposed \textbf{S}tructured \textbf{G}raph learning framework with \textbf{M}ultiple \textbf{K}ernel (SGSK) as:
\begin{equation}
\begin{split}
\min_{Z, P, w} &Tr(K_w-2K_wZ+Z^TK_wZ)+Tr(Z^TD^x)+\alpha \|Z\|_F^2\\
&+\gamma Tr(P^TLP), \\
 \quad s.t.&\quad Z^T\vec{1}=\vec{1}, \quad 0\leq Z\leq 1,\quad P^TP=I,\\
&\quad K_w=\sum\limits_{i=1}^r w_iK^i,\quad\sum\limits_{i=1}^r \sqrt{w_i}=1,\quad w_i\ge 0.
\end{split}
\label{multimodel}
\end{equation}
\subsection{Optimization}
We can iteratively solve $Z, P$, and $ w$, so that each of them will be adaptively refined by the results of the other two.

When $w$ is fixed, we can directly calculate $K_w$, and the optimization problem goes back to (\ref{newmodel}). We can update $Z$ and $P$ by following Algorithm 1 with $K_w$ as the input kernel.

When $Z$ and $P$ are known, solving (\ref{multimodel}) with respect to $w$ can be rewritten as:
\begin{equation}
\label{optie}
\min_w \sum\limits_{i=1}^r w_i h_i  \quad
 s.t.\quad  \sum\limits_{i=1}^r \sqrt{w_i}=1, \quad w_i\ge 0, 
\end{equation}
where 
\begin{equation}
\label{h}
h_i=Tr(K^i-2K^iZ+Z^TK^iZ).
\end{equation}
The Lagrange function corresponding to (\ref{optie}) is 
\begin{equation}
\mathcal{J}(w)=w^Th+g (1-\sum_{i=1}^r\sqrt{w_i}).
\end{equation}
According to the KKT condition, we require $\frac{\partial \mathcal{J}(w)}{\partial w_i}=0$. Then, $w$ has the following expression:
\begin{equation}
\label{weight}
w_i=(h_i \sum_{j=1}^r \frac{1}{h_j})^{-2}.
\end{equation}
In summary, our algorithm for solving (\ref{multimodel}) is provided in Algorithm 2.
\begin{algorithm}
\caption{The algorithm of SGMK}
\label{alg2}
 {\bfseries Input:} Kernel matrices $\{K^i\}_{i=1}^r$, parameter $\gamma>0$, $\alpha$.\\
{\bfseries Initialize:} Random matrix $Z$, $w_i=1/r$.\\
 {\bfseries REPEAT}
\begin{algorithmic}[1]
\STATE Compute $K_\vec{w}$ by (\ref{ukernel}).
 \STATE Calculte $P$ as the $c$ smallest eigenvectors of $L=D-\frac{Z+Z^T}{2}$.
\STATE For each $i$, update the $i$-th column of $Z$ according to (\ref{solveZ}).
\STATE Compute $h$ by (\ref{h}).
\STATE Calculate $w$ by (\ref{weight}).
\end{algorithmic}
\textbf{ UNTIL} {stopping criterion is met.}
\end{algorithm}
\subsection{Extend to Semi-supervised Classification}
Model (\ref{sgsk}) also lends itself to semi-supervised classification. Graph construction and label inference are two fundamental stages in semi-supervised learning (SSL). Solving two separate problems only once is suboptimal since label information is not exploited when learning the graph. SGMK unifies these two fundamental components into a unified framework. Then the given labels and estimated labels will be utilized to build the graph and to predict the unknown labels.

Based on a similar approach, we can reformulate SGMK for semi-supervised classification as:
\begin{equation}
\begin{split}
\min_{Z, P, w} &Tr(K_w-2K_wZ+Z^TK_wZ)+Tr(Z^TD^x)+\alpha \|Z\|_F^2\\
&+\gamma Tr(P^TLP)\\
 \quad s.t.&\quad Z^T\vec{1}=\vec{1}, \quad 0\leq Z\leq 1,\quad P_l=Y_l,\\
&\quad K_w=\sum\limits_{i=1}^r w_iK^i,\quad\sum\limits_{i=1}^r \sqrt{w_i}=1,\quad w_i\ge 0,
\end{split}
\label{ssl}
\end{equation}
where $Y_l=[y_1,\cdots,y_l]^T$ denote the label matrix. $y_i\in\mathcal{R}^{c\times 1}$ and $l$ is the number of labled points. $y_i$ is one-hot and $y_{ij}=1$ indicates that the $i$-th sample belongs to the $j$-th class. (\ref{ssl}) can be solved in the same procedure as (\ref{multimodel}), the only difference is updating $P$. 

For convenience, we rearrange all the points and put the unlabeled $u$ points in the back, e.g., $P=[Y_l; P_u]$. To solve $P$, we take the derivative of (\ref{ssl}) with respect to $P$, we have $LP=0$, i.e.,  
\[
\begin{bmatrix}
L_{ll} &L_{lu}\\
L_{ul} & L_{uu}
\end{bmatrix}  
\begin{bmatrix}
Y_{l} \\
P_{u}
\end{bmatrix}  
=0.
\]
Then $P_u=-L_{uu}^{-1}L_{ul}Y_l$. Finally, the class label for unlabeled points could be assigned according to following decision rule:

\begin{equation}
 y_i=\argmax_j P_{ij}.
\end{equation}

\section{Clustering Experiments}
\label{clusterexperiment}
In this section, we demonstrate the effectiveness of our proposed method on clustering application. 

\captionsetup{position=top}
\begin{table}[!htbp]
\centering
\caption{Description of the data sets}
\label{data}
\renewcommand{\arraystretch}{1.2}
\begin{tabular}{|l|c|c|c|}
\hline
&\textrm{\# instances}&\textrm{\# features}&\textrm{\# classes}\\\hline
\textrm{YALE}&165&1024&15\\\hline
\textrm{JAFFE}&213&676&10\\\hline
\textrm{ORL}&400&1024&40\\\hline
\textrm{AR}&840&768&120\\\hline
\textrm{BA}&1404&320&36\\\hline
\textrm{TR11}&414&6429&9\\\hline
\textrm{TR41}&878&7454&10\\\hline
\textrm{TR45}&690&8261&10\\\hline
\end{tabular}
\end{table}

\subsection{Data Sets}
We implement experiments on eight publicly available data sets. The statistics information of these data sets is summarized in Table \ref{data}. Specifically, the first five data sets include four face databases (ORL\footnote{http://www.cl.cam.ac.uk/research/dtg/attarchive/facedatabase.html}, YALE\footnote{http://vision.ucsd.edu/content/yale-face-database}, AR\footnote{http://www2.ece.ohio-state.edu/ aleix/ARdatabase.html}, and JAFFE\footnote{http://www.kasrl.org/jaffe.html}) and a binary alpha digits data set BA\footnote{http://www.cs.nyu.edu/~roweis/data.html}. Tr11, Tr41, and Tr45 are derived from NIST TREC Document Database\footnote{http://www-users.cs.umn.edu/{\raise.17ex\hbox{$\scriptstyle\sim$}}han/data/tmdata.tar.gz}. 

Following the setting in \cite{du2015robust}, we design 12 kernels. They are: seven Gaussian kernels of the form $K(x,y)=exp(-\|x-y\|_2^2/(td_{max}^2))$, where $d_{max}$ is the maximal distance between samples and $t$ varies over the set $\{0.01, 0.0, 0.1, 1, 10, 50, 100\}$; a linear kernel $K(x,y)=x^\top y$; four polynomial kernels $K(x,y)=(a+x^\top y)^b$ with $a\in\{0,1\}$ and $b\in\{2,4\}$. Besides, all kernels are rescaled to $[0,1]$ by dividing each element by the largest pairwise squared distance. 
\captionsetup{position=top}
\begin{table*}[!ht]
\centering
\small
\renewcommand{\arraystretch}{1.4}
\caption{Clustering results of various methods. The average performance of those 12 kernels are put in parenthesis. Single and multiple kernel methods are separated by double lines. The best performance of single and multiple kernel methods are highlighted in boldface. ‘-’ denotes the results are unavailable due to numerical error (text data is sparse).\label{clusterres}}
\subfloat[Accuracy(\%)\label{acc}]{
\resizebox{\textwidth}{!}{
\begin{tabular}{|l  |c |c| c|c|c|c| c | |c| c| c| c| c}
\hline
        \tiny{Data} & \tiny{SC}  &\tiny{RKKM}&\tiny{LRR} &\tiny{SSR} & \tiny{Local}&\tiny{Global}& \tiny{SGSK}&\tiny{MKKM} & \tiny{AASC} & \tiny{RMKKM}&\tiny{SGMK} \\
     \hline
        \multirow{1}{*}{\tiny{YALE}}  &49.42(40.52)&48.09(39.71)&53.94&54.55& 58.79&55.85(45.35)&\textbf{62.75}(62.05)&45.70&40.64&52.18&\textbf{63.62}\\
	
		\hline
		\multirow{1}{*}{\tiny{JAFFE}}  &74.88(54.03)&75.61(67.98)&70.89&87.32&98.12&\textbf{99.83}(86.64)&99.53(98.12)&74.55&30.35&87.07&\textbf{99.53}\\
		
		\hline
        \multirow{1}{*}{\tiny{ORL}}  &57.96(46.65)&54.96(46.88)&\textbf{71.50}&69.00& 61.50&62.35(50.50)&70.05(62.10)&47.51&27.20&55.60&\textbf{70.02}\\
		
		\hline
        \multirow{1}{*}{\tiny{AR}}  &28.83(22.22)&33.43 (31.20)&32.02&\textbf{65.00}&42.26&56.79(41.35)&62.59(48.21)&28.61&33.23&34.37&\textbf{63.45}\\
		
%
		\hline
        \multirow{1}{*}{\tiny{BA}}&31.07(26.25)&42.17(34.35)&25.93&23.97& 36.82&47.72(39.50)&\textbf{48.32}(37.59)&40.52&27.07&43.42&\textbf{49.37}\\
      
       \hline
        \multirow{1}{*}{\tiny{TR11}}  &50.98(43.32)&53.03(45.04)&-&41.06& 38.89&71.26(54.88)&\textbf{71.74}(54.92)&50.13&47.15&57.71&\textbf{74.40}\\
		
		\hline
		\multirow{1}{*}{\tiny{TR41}} &63.52(44.80)&56.76(46.80)&-&63.78 &62.89&67.43(53.13)&\textbf{72.67}(69.15)&56.10&45.90&62.65&\textbf{79.38}\\
		
		\hline
		\multirow{1}{*}{\tiny{TR45}}&57.39(45.96)&58.13(45.69)&-&71.45 & 56.96&74.02(53.38)&\textbf{77.54}(75.33)&58.46&52.64&64.00&\textbf{77.54}\\
		
		\hline
\end{tabular}}

}\\
\renewcommand{\arraystretch}{1.4}
\subfloat[NMI(\%)\label{NMI}]{
\resizebox{\textwidth}{!}{
\begin{tabular}{|l  |c |c| c|c|c|c| c | |c| c| c| c| c}
	\hline
    \tiny{Data}   & \tiny{SC} &\tiny{RKKM} &\tiny{LRR}& \tiny{SSR} & \tiny{Local}&\tiny{Global}& \tiny{SGSK}  &\tiny{MKKM} & \tiny{AASC} & \tiny{RMKKM}&\tiny{SGMK} \\
	\hline
        \multirow{1}{*}{\tiny{YALE}}&52.92(44.79)&52.29(42.87)&59.39&57.26 & 57.67&56.50(45.07)&\textbf{61.58}(60.47)&50.06&46.83&55.58&\textbf{62.04}\\
		
		\hline
		\multirow{1}{*}{\tiny{JAFFE}} &82.08(59.35)&83.47(74.01)&75.73&92.93&97.31&\textbf{99.35}(84.67)&99.18(97.62)&79.79&27.22&89.37&\textbf{99.18}\\
	\hline
        \multirow{1}{*}{\tiny{ORL}} &75.16(66.74)&74.23(63.91)&\textbf{85.40}&84.23& 76.59&78.96(63.55)&82.65(75.93)&68.86&43.77&74.83&\textbf{81.94}\\
		
	\hline
        \multirow{1}{*}{\tiny{AR}} &58.37(56.05)&65.44 (60.81)&67.23&\textbf{84.16}& 65.73&76.02(59.70)&82.61(67.63)&59.17&65.06&65.49&\textbf{83.51}\\
		
		
	\hline
        \multirow{1}{*}{\tiny{BA}}  &50.76(40.09)&57.82(46.91)&40.74&30.29&49.32&\textbf{63.04}(52.17)&61.94(52.71)&56.88&42.34&58.47&\textbf{62.25}\\
	\hline
        \multirow{1}{*}{\tiny{TR11}}  &43.11(31.39)&49.69(33.48)&-&27.60& 19.17&58.60(37.58)&\textbf{62.07}(38.98)&44.56&39.39&56.08&\textbf{64.18}\\
		\hline
		\multirow{1}{*}{\tiny{TR41}} &61.33(36.60)&60.77(40.86)&-&59.56& 51.13&65.50(43.18)&\textbf{70.59}(63.67)&57.75&43.05&63.47&\textbf{69.85}\\
		\hline
		\multirow{1}{*}{\tiny{TR45}}&48.03(33.22)&57.86(38.96)&-&67.82&49.31&\textbf{74.24}(44.36)&70.7(69.70)&56.17&41.94&62.73&\textbf{70.92}\\
	\hline
\end{tabular}}
}\\
\renewcommand{\arraystretch}{1.4}
\subfloat[ Purity(\%)\label{purity}]{
\resizebox{\textwidth}{!}{
\begin{tabular}{|l  |c |c| c|c|c|c| c | |c| c| c| c| c}
	\hline
       \tiny{Data} & \tiny{SC} &\tiny{RKKM} &\tiny{LRR}& \tiny{SSR} & \tiny{Local}&\tiny{Global}  & \tiny{SGSK} &\tiny{MKKM} & \tiny{AASC} & \tiny{RMKKM}&\tiny{SGMK} \\
   	\hline
        \multirow{1}{*}{\tiny{YALE}}  &51.61(43.06)&49.79(41.74)&55.15&58.18&59.39&57.27(55.79)&\textbf{66.77}(66.19)&47.52&42.33&53.64&\textbf{67.79}\\
	\hline
		\multirow{1}{*}{\tiny{JAFFE}} &76.83(56.56)&79.58(71.82)&74.18&96.24 &98.12&\textbf{99.85}(96.53)&99.53(98.17)&76.83&33.08&88.90&\textbf{99.53}\\
	\hline
        \multirow{1}{*}{\tiny{ORL}}  &61.45(51.20)&59.60(51.46)&75.25&76.50&\textbf{76.59}&74(70.37)&75.35(71.62)&52.85&31.56&60.23&\textbf{77.00}\\
		\hline
        \multirow{1}{*}{\tiny{AR}}  &33.24(25.99)&35.87 (33.88)&33.33&69.52&44.64&63.45(62.37)&\textbf{80.60}(62.54)&30.46&34.98&36.78&\textbf{83.57} \\
		\hline
        \multirow{1}{*}{\tiny{BA}} &34.50(29.07)&45.28(36.86)&28.70&40.85& 39.67&52.36(49.79)&\textbf{57.36}(55.74)&43.47&30.29&46.27&\textbf{58.27}\\
	\hline
        \multirow{1}{*}{\tiny{TR11}}  &58.79(50.23)&67.93(56.40)&-&\textbf{85.02}& 44.20&82.85(80.76)&81.40(80.07)&65.48&54.67&72.93&\textbf{82.37}\\
	\hline
		\multirow{1}{*}{\tiny{TR41}}&73.68(56.45)&74.99(60.21)&-&75.40 &67.54&73.23(71.21)&\textbf{78.36}(77.19)&72.83&62.05&77.57&\textbf{87.13}\\
		\hline
		\multirow{1}{*}{\tiny{TR45}}&61.25(50.02)&68.18(53.75)&-&\textbf{83.62} & 60.87&78.26(77.76)&78.70(78.06)&69.14&57.49&75.20&\textbf{78.70}\\
	\hline
\end{tabular}
}}

\end{table*}

\subsection{Comparison Methods}
To fully investigate the performance of our method on clustering, we choose a good set of methods to compare.
\begin{itemize}
\item{\textbf{Spectral Clustering (SC) }\cite{ng2002spectral}: SC is a widely used clustering technique. It enjoys the advantage of exploring the intrinsic data structures. However, how to construct a good similarity graph is an open issue. Here, we directly use kernel matrix as its input.}
\item{\textbf{Robust Kernel K-means (RKKM)}\cite{du2015robust}: As an extension to classical k-means clustering method, RKKM has the capability of dealing with nonlinear structure, noise, and outliers in the data, since $\ell_{21}$-norm is adopted to measure the loss of k-means. RKKM shows promising results on a number of real-world data sets.}
\item{\textbf{Low Rank Representation (LRR)} \cite{liu2013robust}: Based on self-expressive property, a low-rak  graph is obtained. }
\item{\textbf{Simplex Sparse Representation (SSR)} \cite{huang2015new}: Based on self-expressive property, a sparse graph is obtained. SSR achieves satisfying performance in numerous data sets.}
\item{\textbf{Local structure learning approach (Local)} \cite{nie2014clustering}: By using adaptive neighbor idea, this method considers local structure (\ref{local}) and the rank constraint. }
\item{\textbf{Global structure learning approach (Global)} \cite{kang2017twin}: Based on self-expressive property, this method incorporates global structure (\ref{global}) and the rank constraint. }
\item{Our proposed \textbf{SGSK} and \textbf{SGMK} methods: Our method combines both local and global structure information. The code for our method is publicly available \footnote{https://github.com/sckangz/ICDE}. }
\item{\textbf{Multiple Kernel K-means (MKKM)} \cite{huang2012multiple}: It is an extension of
k-means in a multiple-kernel setting. Besides, a different way of kernel weight learning is used.}
\item{\textbf{Affinity Aggregation for Spectral Clustering (AASC)} \cite{huang2012affinity}: It is a version of 
spectral clustering where multiple affinity graphs exist.}
\item{\textbf{Robust Multiple Kernel K-means (RMKKM)} \cite{du2015robust}: It extends RKKM to the situation of multiple kernels. }
\end{itemize}

\subsection{Clustering Results}
To quantitatively assess the performance of our proposed method, we adopt the commonly used metrics, accuracy (Acc), normalized mutual information (NMI), and Purity \cite{peng2018integrate}. 
We present the experimental results of different methods in Table \ref{clusterres}. We can see that our proposed methods obtain promising results. More precisely, we have the following observations.
\begin{itemize}
\item{Compared to traditional spectral clustering and recently proposed robust kernel k-means techniques, our method can enhance the performance considerably. For instance, in terms of the best acc, SGSK improves over SC and RMMK by 42.95\%, 36.01\% on average, respectively.}
\item{Adaptive neighbor and self-expressiveness based approaches outperform spectral clustering and k-means based methods. Specifically, LRR, SSR, Local, Global, SGSK perform much better than SC and RKKM on YALE, JAFFE, ORL, AR datasets. Among them, SGSK, which combines the complementary information carried by local and global structure, works the best in most cases. This confirms the equally importance of local and global structure information.}
\item{For multiple kernel learning based methods, our proposed SGMK achieves much better results than MKKM, AASC, RMKKM. Furthermore, the performance of multiple kernel methods are close to or better than their corresponding single kernel methods.}

\end{itemize}
\subsection{Ablation Study}
The Local and Global results in Table \ref{clusterres} have demonstrate the importance of local and global strcture learning. Here we further investigate their importance in the multiple kernel setting. In particular, we show the results of SGMK and SGMK without local structure part in Table \ref{abl}.

Once again, we can observe that our global and local structure unified model generally outperforms the model only with global structure learning. This strongly verifies the benefit of incorporating both global and local structure in graph learning. Furthermore, it can be seen that global part can obtain better performance than SGMK in several cases. This could be caused by the fact that we treat the global and local structure terms equally important in our model (\ref{multimodel}). In real-world applications, global structure might be more improtant than local structure in some data sets. In such cases, it would be more practical to introduce a parameter to balance the first two terms in Eq. (\ref{multimodel}).
\begin{table*}[ht]
\centering
\caption{Global and local structure effect in multiple kernel learning setting.}
\renewcommand{\arraystretch}{1.1}
\resizebox{.8\textwidth}{!}{
\begin{tabular}{|c |c|c| c |c| c |c |c |c| c|}
\hline
Metric & Method & YALE & JAFFE & ORL & AR & BA & TR11 & TR41 & TR45 \\
\hline
& SGMK without Local& 56.97 & 100 & 65.25 & 62.38 & 47.34& 73.43 &67.31 & 74.35\\
 Acc &SGMK& 63.62 & 99.53 & 70.02 & 63.45 & 49.37 & 74.40 & 79.38 & 77.54\\
 \hline \hline
&SGMK without Local &56.52 & 100 & 80.04 & 81.51& 62.94 & 60.15 & 65.11 & 74.97\\
 NMI&SGMK  & 62.04 & 99.18 & 81.94 & 83.51& 62.25 & 64.18 & 69.85 & 70.92\\
 \hline \hline
 &SGMK without Local & 60.00 & 100 & 77.00 & 82.62 & 52.12 & 87.44 & 73.69 & 78.26\\
 Purity&SGMK  & 67.79 & 99.53 & 77.00 & 83.57 & 58.27 & 82.37 & 87.13 & 78.70\\
 
\hline
\end{tabular}}

\label{abl}
\end{table*}

\subsection{Parameter Sensitivity}
There are two parameters in our model: $\alpha$ and $\gamma$. As we discussed in subsection \ref{choosealpha}, the search for $\alpha$ can be better handled by searching for a proper neighborhood size $k$. Therefore, we perform grid search for the $\gamma$ and $k$ that produce the best performance. Taking YALE and JAFFE data sets as examples, we demonstrate the sensitivity of our model SGMK to $\gamma$ and $k$ in Figure \ref{yalepara} and \ref{jaffepara}. They illustrate that our method works well $\gamma$ and $k$ over wide ranges of values. For $k$, we can increase its value when there are more samples in the data set.

\captionsetup{position=bottom}
\begin{figure*}[!htbp]
\centering
\subfloat[Accuracy\label{acc}]{\includegraphics[width=.33\textwidth]{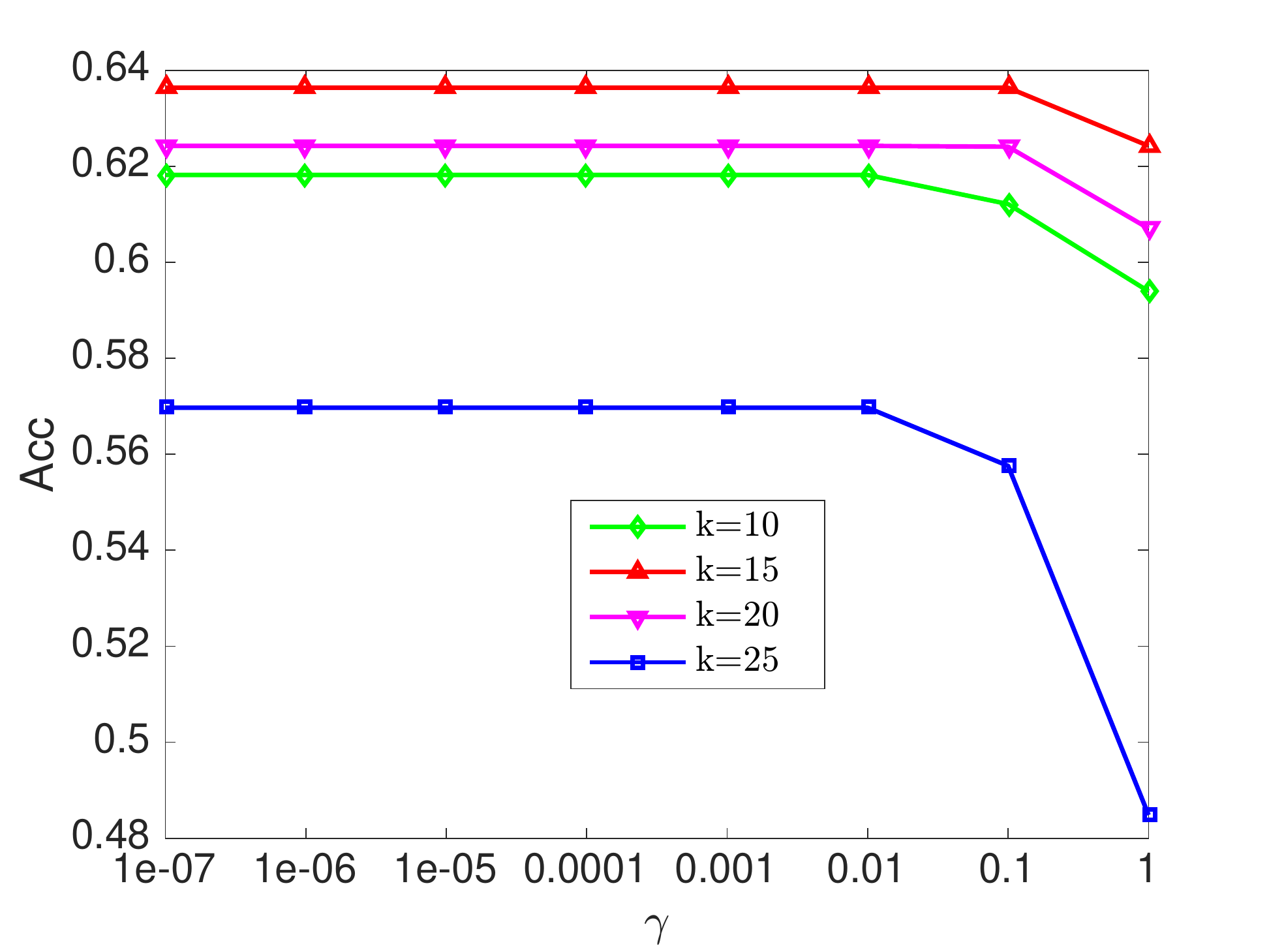}}
\subfloat[NMI\label{nmi}]{\includegraphics[width=.33\textwidth]{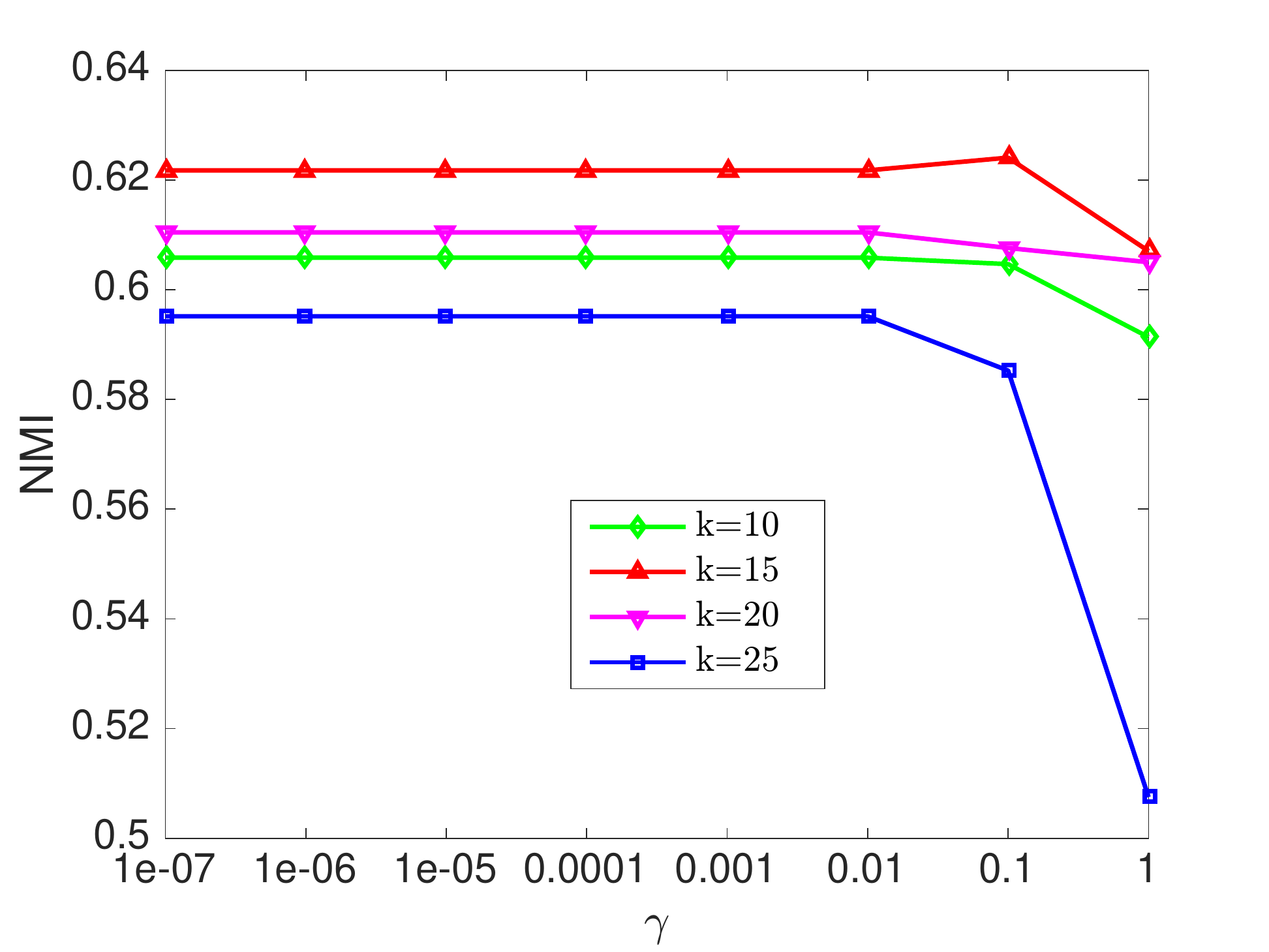}}
\subfloat[Purity\label{Purity}]{\includegraphics[width=.33\textwidth]{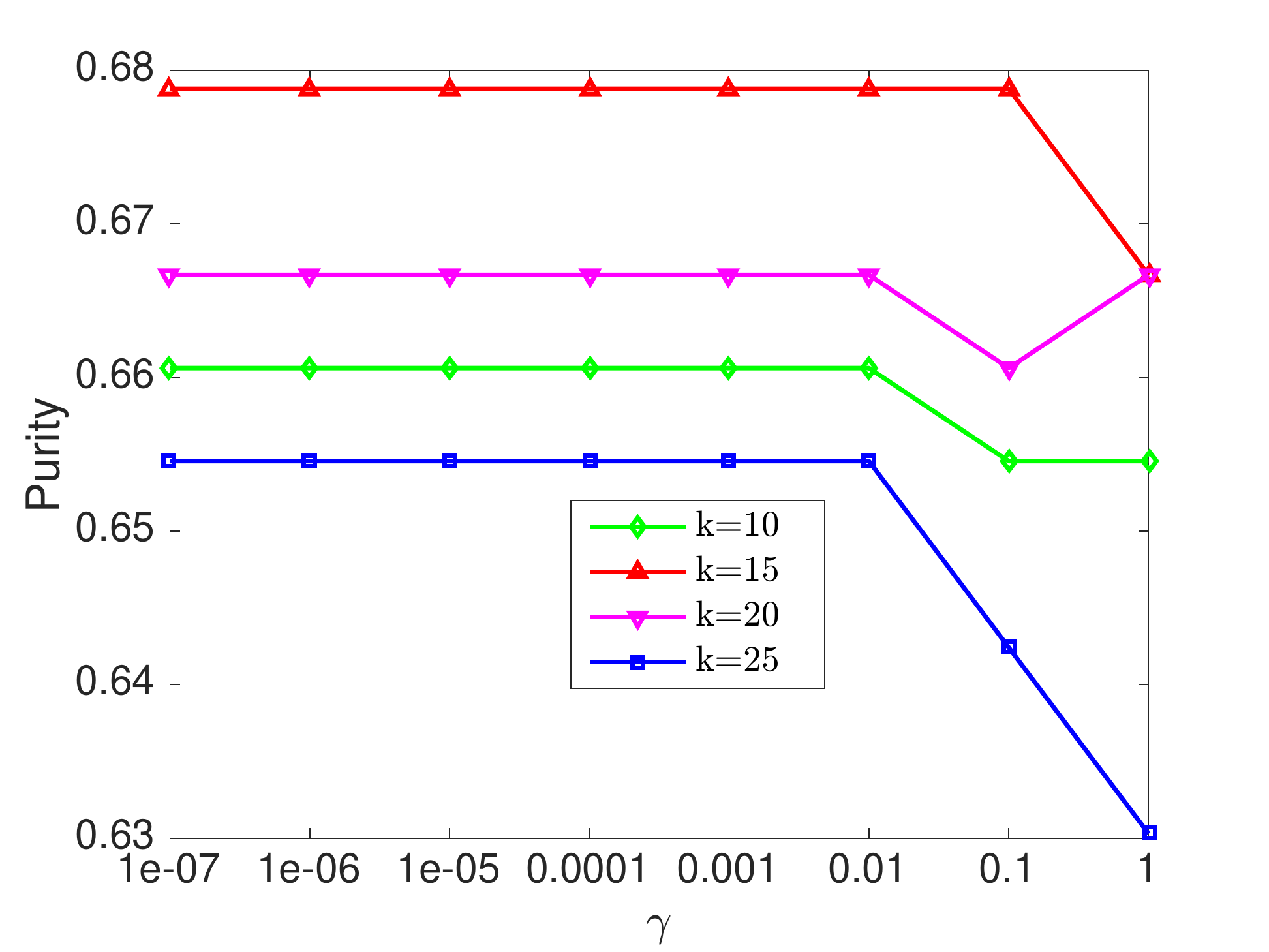}}
\caption{Parameter influence on YALE data set.\label{yalepara}}
\end{figure*}
\begin{figure*}
\centering
\subfloat[Accuracy\label{acc}]{\includegraphics[width=.33\textwidth]{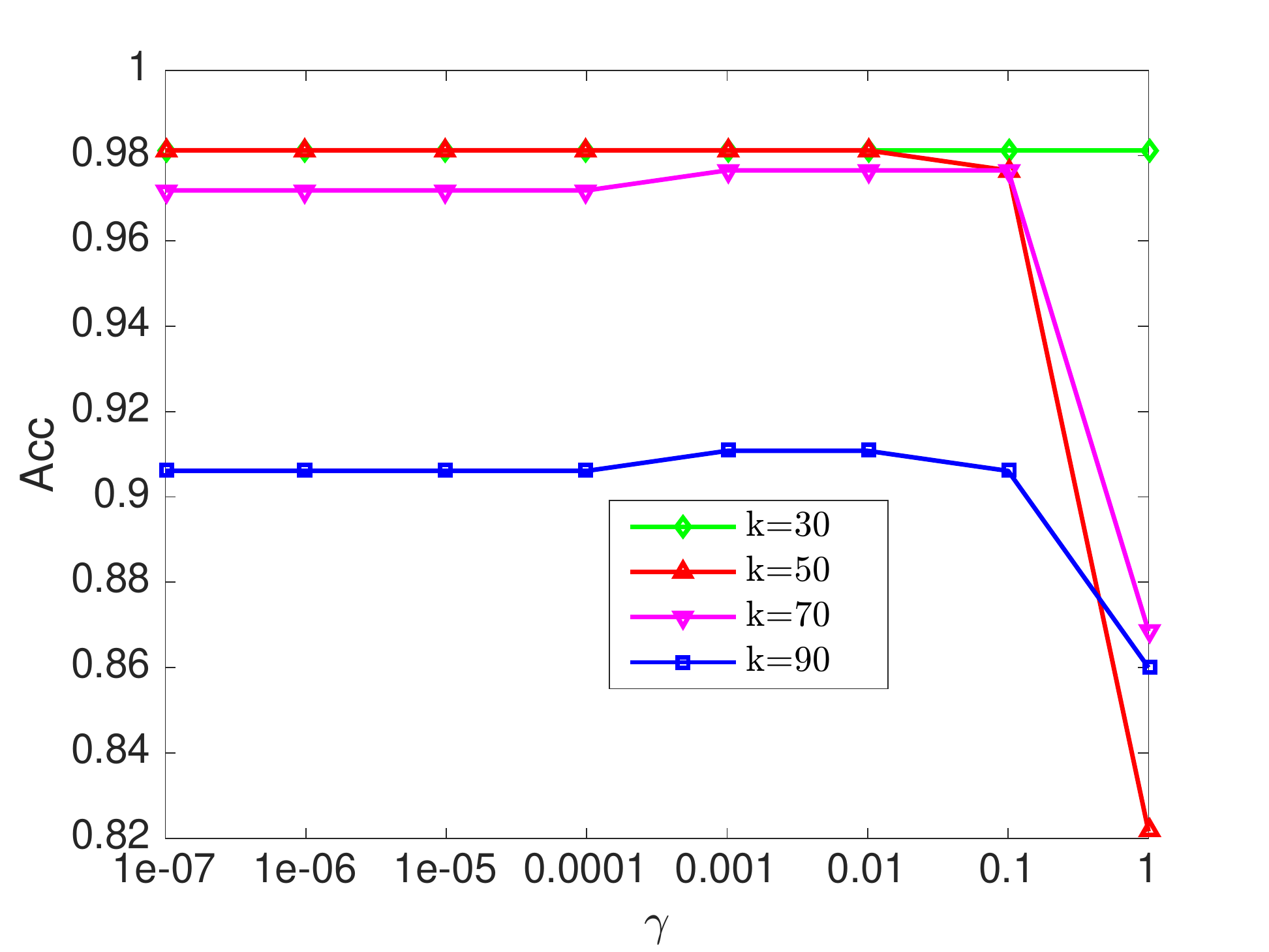}}
\subfloat[NMI\label{nmi}]{\includegraphics[width=.33\textwidth]{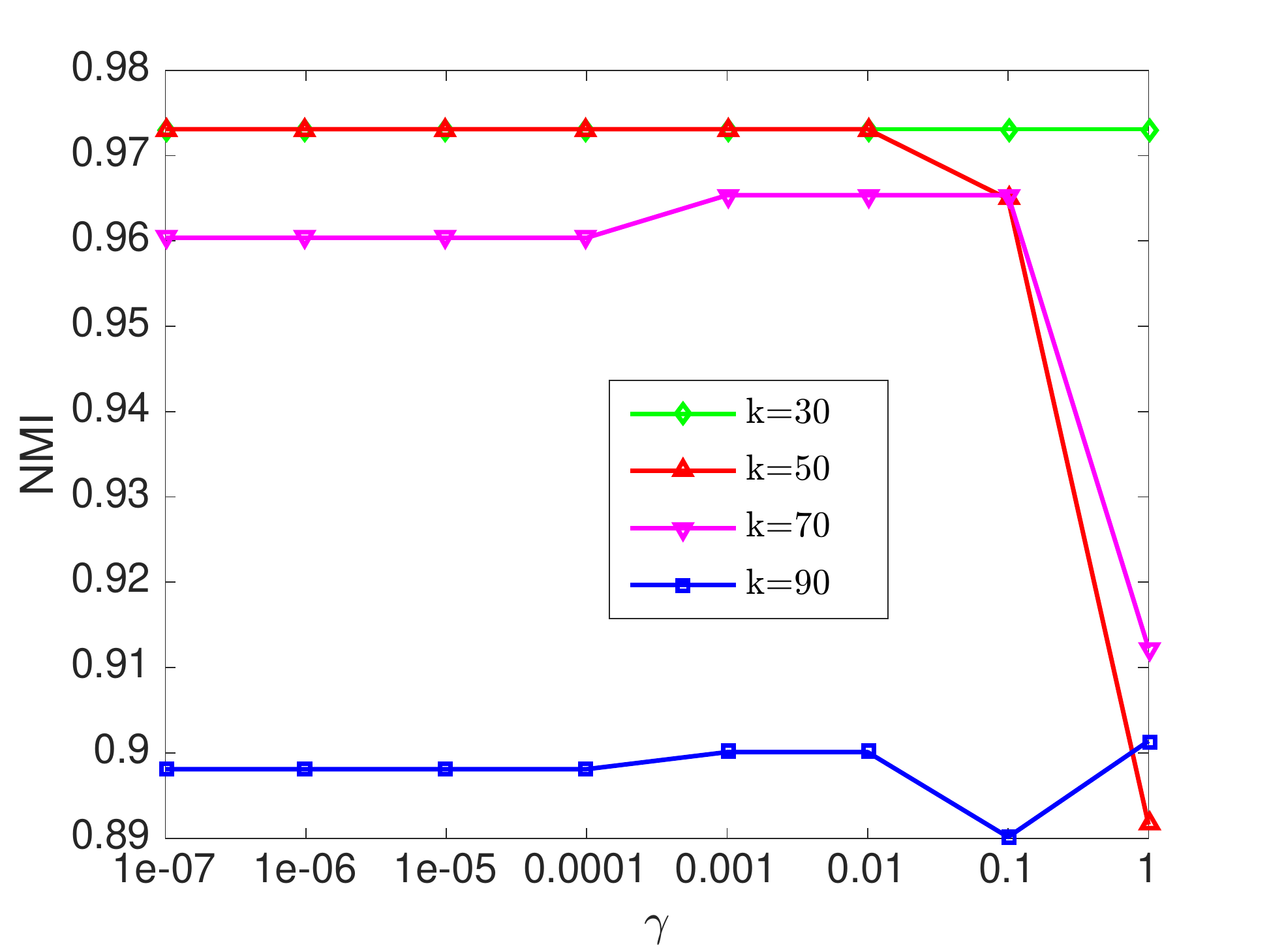}}
\subfloat[Purity\label{Purity}]{\includegraphics[width=.33\textwidth]{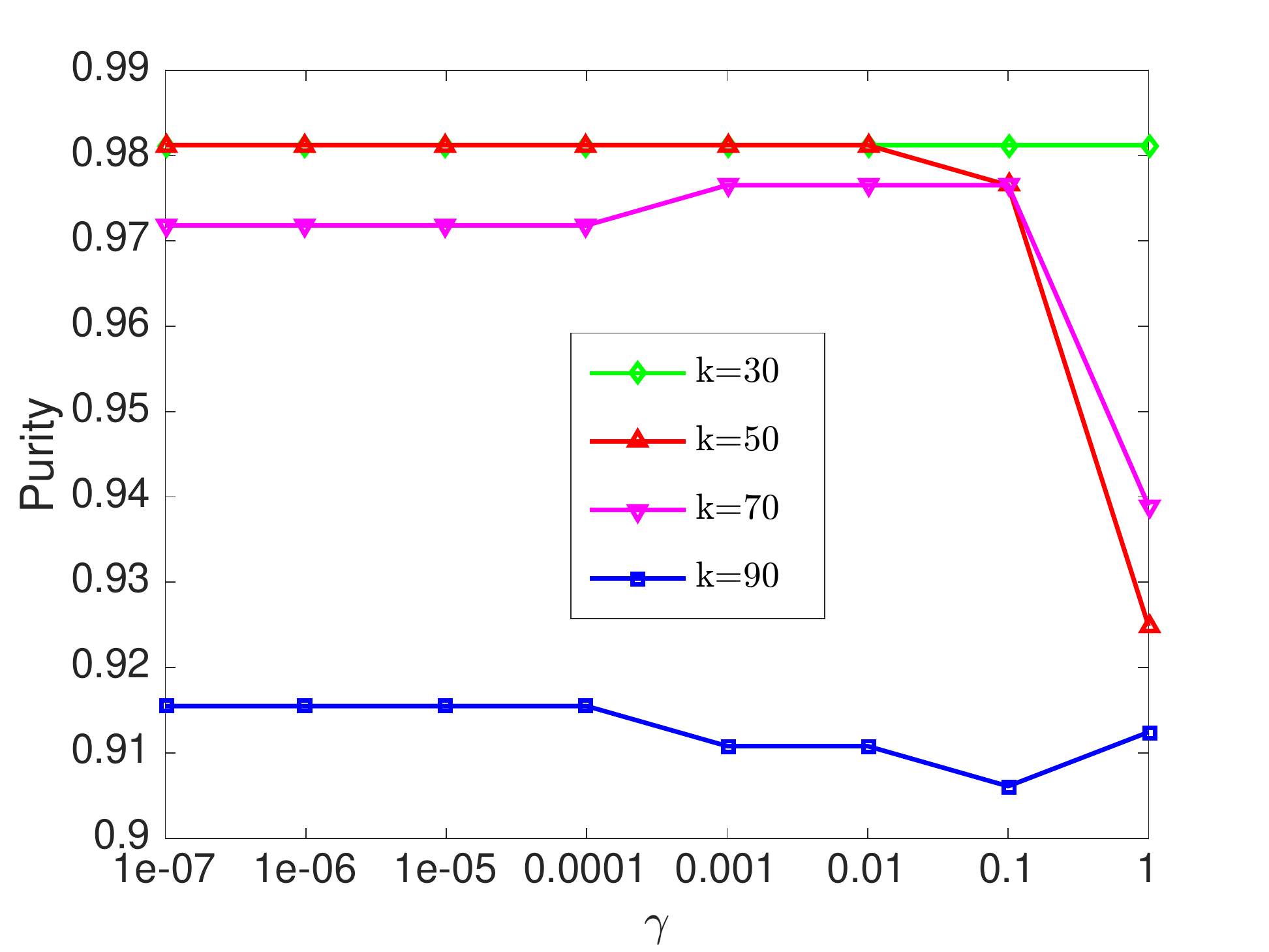}}
\caption{Parameter influence on JAFFE data set.\label{jaffepara}}
\end{figure*}

\begin{figure*}[!htbp]
\centering
\subfloat[YALE\label{yale}]{\includegraphics[width=.28\textwidth]{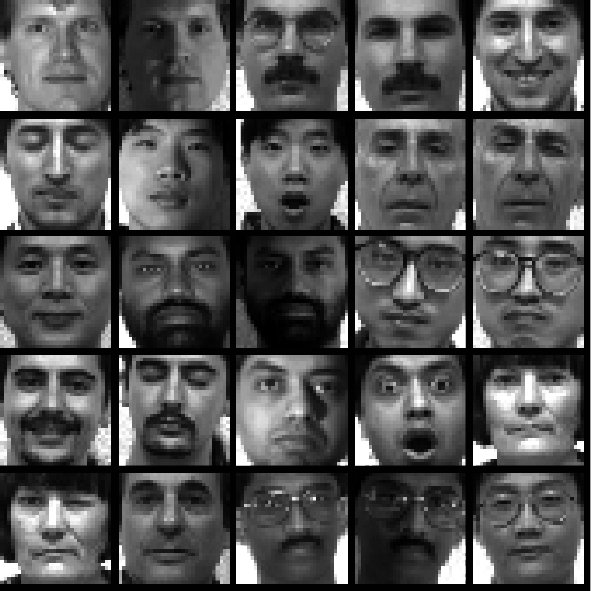}}
\subfloat[BA\label{ba}]{\includegraphics[width=.3\textwidth]{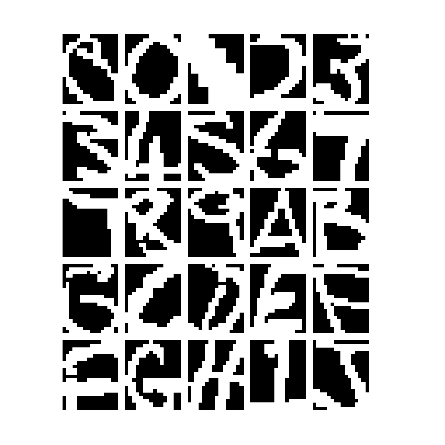}}
\subfloat[COIL20\label{coil20}]{\includegraphics[width=.34\textwidth]{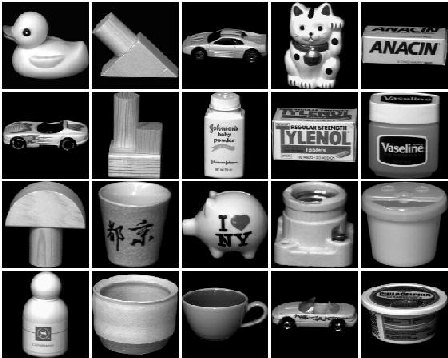}}
\caption{Sample images of YALE, BA, and  COIL20.}
\end{figure*}

\section{Semi-supervised Classification Experiments}
\label{semiexperiment}
In this section, we assess the effectiveness of SGMK on semi-supervised learning (SSL) task.
\subsection{Data Sets}
1) \textbf{Evaluation on Face Recognition}: We examine the effectiveness of our graph learning for face recognition on two frequently used face databases: YALE and JEFFE. The YALE face data set contains 15 individuals, and each person has 11 near frontal images taken under different illuminations. Each image is resized to 32$\times$32 pixels. Some sample images are shown in Figure \ref{yale}. The JAFFE face database consists of 10 individuals, and each subject has 7 different facial expressions (6 basic facial expressions +1 neutral). The images are resized to 26$\times$26 pixels. \\
2) \textbf{Evaluation on Digit/Letter Recognition}: In this experiment, we address the digit/letter recognition problem on the BA database. The data set consists of digits of ``0" through ``9" and letters of capital ``A" to ``Z". Therefore, there are 39 classes and each class has 39 samples. Figure \ref{ba} shows some sample images from BA database.\\
3) \textbf{Evaluation on Visual Object Recognition}: We conduct visual object recognition experiment on the COIL20 database. The database consists of 20 objects and 72 images for each object. For each object, the images were taken 5 degrees apart as the object is rotating on a turntable.  The size of each image is 32$\times$32 pixels. Some sample images are shown in Figure \ref{coil20}. \\
Similar to clustering experiment, we construct 7 kernels for each data set. They include: four Gaussian kernels with $t$ varies over $\{0.1, 1, 10, 100\}$; a linear kernel $K(x,y)=x^\top y$; two polynomial kernels $K(x,y)=(a+x^\top y)^2$ with $a\in\{0,1\}$.

\begin{table*}[htbp]
\begin{center}
\caption{Classification accuracy (\%) on benchmark data sets (mean$\pm$standard deviation). The best results are in bold font.\label{classres}}
\renewcommand{\arraystretch}{1.4}
\resizebox{.99\textwidth}{!}{
\begin{tabular}{ |c|c|c|c|c|c|c|c|}
\hline
Data &Labeled Percentage($\%$) &GFHF & LGC &S$^3$R&S$^2$LRR& SCAN &SGMK\\
\hline\hline
\multirow{3}{4em}{YALE} & 10 &38.00$\pm$11.91&47.33$\pm$13.96& 38.83$\pm$8.60 &28.77$\pm$9.59& 45.07$\pm$1.30 &\textbf{52.40}$\pm$0.19\\ 
& 30 & 54.13$\pm$9.47&63.08$\pm$2.20& 58.25$\pm$4.25& 42.58$\pm$5.93& 60.92$\pm$4.03&\textbf{75.58}$\pm$0.04\\ 
& 50 & 60.28$\pm$5.16&69.56$\pm$5.42& 69.00$\pm$6.57& 51.22$\pm$6.78 & 68.94$\pm$4.57& \textbf{82.11}$\pm$0.05\\ 
\hline
\multirow{3}{4em}{JAFFE} & 10 & 92.85$\pm$7.76&96.68$\pm$2.76& 97.33$\pm$1.51& 94.38$\pm$6.23& 96.92$\pm$1.68 & \textbf{99.57}$\pm$0.02\\ 
& 30 &98.50$\pm$1.01&98.86$\pm$1.14& 99.25$\pm$0.81& 98.82$\pm$1.05& 98.20$\pm$1.22& \textbf{99.90}$\pm$0.01\\ 
& 50 &98.94$\pm$1.11&99.29$\pm$0.94& 99.82$\pm$0.60& 99.47$\pm$0.59 & 99.25$\pm$5.79& \textbf{100}$\pm$0.00\\ 
\hline\hline
\multirow{3}{4em}{BA} & 10 &45.09$\pm$3.09&48.37$\pm$1.98& 25.32$\pm$1.14 &20.10$\pm$2.51&55.05$\pm$1.67& \textbf{58.77}$\pm$0.83\\ 
& 30 &62.74$\pm$0.92&63.31$\pm$1.03& 44.16$\pm$1.03& 43.84$\pm$1.54&68.84$\pm$1.09&\textbf{89.88}$\pm$0.27\\ 
& 50 &68.30$\pm$1.31&68.45$\pm$1.32& 54.10$\pm$1.55& 52.49$\pm$1.27&72.20$\pm$1.44& \textbf{90.60}$\pm$0.13\\ 
\hline\hline
\multirow{3}{4em}{COIL20} & 10 &87.74$\pm$2.26&85.43$\pm$1.40& \textbf{93.57}$\pm$1.59& 81.10$\pm$1.69&90.09$\pm$1.15 & 90.74$\pm$0.64\\ 
& 30 &95.48$\pm$1.40&87.82$\pm$1.03&96.52$\pm$0.68& 87.69$\pm$1.39 &95.27$\pm$0.93&\textbf{96.85}$\pm$0.32\\ 
& 50 &96.27$\pm$0.71&88.47$\pm$0.45&97.87$\pm$0.10& 90.92$\pm$1.19 &97.53$\pm$0.82& \textbf{98.74}$\pm$0.08\\ 
\hline
\end{tabular}}
\end{center}

\end{table*}

\subsection{Comparison Methods}
We compare our method with several other state-of-the-art algorithms.
\begin{itemize}
\item {\textbf{ Local and Global Consistency (LGC)} \cite{zhou2004learning}: LGC is a popular label propagation method. For this method, kernel matrix is used to compute $L$. }
\item{\textbf{Gaussian Field and Harmonic function (GFHF)} \cite{zhu2003semi}: Different from LGC, GFHF is another mechanics to infer those unknown labels as a process of propagating labels through the pairwise similarity.}
\item{\textbf{Semi-supervised Classification with Adaptive Neighbours (SCAN)} \cite{nie2017multi}: Based on adaptive neighbors method, SCAN adds the rank constraint to ensure that $Z$ has exact $c$ connected components. As a result, the similarity matrix and class indicator matrix $F$ are learned simultaneously. It shows much better performance than many other techniques.  }
\item{\textbf{A Unified Optimization Framework for Semisupervised Learning} \cite{li2015learning}: Li et al. propose a unified framework based on self-expressiveness approach. Similar to SCAN, the similarity matrix and class indicator matrix $F$ are updated alternatively. By using low-rank and sparse regularizer, they have S$^2$LRR and S$^3$R method, respectively.}
\item{Our Proposed \textbf{SGMK}: SGMK integrates both local and global structure information, with a rank constraint to improve the quality of graph.}
\end{itemize}
\subsection{Classification Results}
We randomly choose some portions of samples as labeled data and repeat 20 times. In our experiment, 10$\%$, 30$\%$, 50$\%$ of samples in each class are randomly selected and labeled. Then, classification accuracy and deviation are shown in Table \ref{classres}. For GFHF and LGC, the aforementioned seven kernels are tested and the best performance is reported. For these two methods, more importantly, the label information is only used in the label propagation stage. For SCAN, S$^2$LRR, S$^3$R, and SGMK, the label prediction and graph learning are conducted in a unified framework, which often leads to better performance.

As expected, the classification accuracy for all methods monotonically increase with the increase of the percentage of labeled samples. As can be observed, our SGMK method outperforms other state-of-the-art methods in general. This confirms the effectiveness of our proposed method on SSL task. Remember that S$^2$LRR and S$^3$R are using self-expressive property to capture the global information, while SCAN is developed to reveal local structure information, so the advantages of our SGMK method over them verify the necessity of incorporating both global and local structure information.

\section{Conclusion}
\label{conclusion}
In this paper, we propose a new graph learning framework by iteratively learning the graph matrix and the labels. Specifically, both local and global structure information is incorporated in our model. We also consider rank constraint on the graph Laplacian, to yield an optimal graph for clustering and classification tasks, so the achieved graph is more informative and discriminative. This turns out to be a unifed model for both graph and label learning, both are improved collaboratively. A multiple kernel learning method is also developed to avoid extensive search for the most suitable kernel. Extensive experiments show the high potential of our method on real-world applications. 

Though impressive performance is achieved, the proposed approach has a high time complexity. In the future, we plan to improve its computation efficiency. This can be addressed by borrowing the idea of anchor point. Specifically, we only need to learn a graph between the whole data points and some landmarks. Considering the crucial role of graph in many algorithms, researchers from many other communities could benefit from this line of research.

  \section*{Acknowledgment}

This paper was in part supported by Grants from the National Key R\&D Program of
China (No. 2018YFC0807500), the Natural Science
Foundation of China (Nos. 61806045, U19A2059), the Sichuan Science and
Techology Program under Project 2020YFS0057, the Ministry of Science and
Technology of Sichuan Province Program (Nos. 2018GZDZX0048, 20ZDYF0343),  the Fundamental Research Fund for the Central
Universities under Project ZYGX2019Z015.

\section{References}
\bibliographystyle{elsarticle-num}

\bibliography{ref}

\begin{thebibliography}{10}
\expandafter\ifx\csname url\endcsname\relax
  \def\url#1{\texttt{#1}}\fi
\expandafter\ifx\csname urlprefix\endcsname\relax\def\urlprefix{URL }\fi
\expandafter\ifx\csname href\endcsname\relax
  \def\href#1#2{#2} \def\path#1{#1}\fi

\bibitem{li2015robust}
Z.~Li, J.~Liu, J.~Tang, H.~Lu, Robust structured subspace learning for data
  representation, IEEE transactions on pattern analysis and machine
  intelligence 37~(10) (2015) 2085--2098.

\bibitem{huang2019auto}
S.~Huang, Z.~Kang, I.~W. Tsang, Z.~Xu, Auto-weighted multi-view clustering via
  kernelized graph learning, Pattern Recognition 88 (2019) 174--184.

\bibitem{yan2007graph}
S.~Yan, D.~Xu, B.~Zhang, H.-J. Zhang, Q.~Yang, S.~Lin, Graph embedding and
  extensions: A general framework for dimensionality reduction, IEEE
  transactions on pattern analysis and machine intelligence 29~(1) (2007)
  40--51.

\bibitem{zhang2013graph}
Z.~Zhang, M.~Zhao, T.~W. Chow, Graph based constrained semi-supervised learning
  framework via label propagation over adaptive neighborhood, IEEE Transactions
  on Knowledge and Data Engineering 27~(9) (2013) 2362--2376.

\bibitem{shuman2013emerging}
D.~I. Shuman, S.~K. Narang, P.~Frossard, A.~Ortega, P.~Vandergheynst, The
  emerging field of signal processing on graphs: Extending high-dimensional
  data analysis to networks and other irregular domains, IEEE Signal Processing
  Magazine 30~(3) (2013) 83--98.

\bibitem{shen2017compressed}
X.~Shen, W.~Liu, I.~Tsang, F.~Shen, Q.-S. Sun, Compressed k-means for
  large-scale clustering, in: Thirty-First AAAI Conference on Artificial
  Intelligence, 2017.

\bibitem{huang2020auto}
S.~Huang, Z.~Kang, Z.~Xu, Auto-weighted multi-view clustering via deep matrix
  decomposition, Pattern Recognition 97 (2020) 107015.

\bibitem{li2015learning}
C.-G. Li, Z.~Lin, H.~Zhang, J.~Guo, Learning semi-supervised representation
  towards a unified optimization framework for semi-supervised learning, in:
  Proceedings of the IEEE International Conference on Computer Vision, 2015,
  pp. 2767--2775.

\bibitem{kang2020robust}
Z.~Kang, H.~Pan, S.~C. Hoi, Z.~Xu, Robust graph learning from noisy data, IEEE
  Transactions on Cybernetics 50~(5) (2020) 1833--1843.

\bibitem{ng2002spectral}
A.~Y. Ng, M.~I. Jordan, Y.~Weiss, et~al., On spectral clustering: Analysis and
  an algorithm, Advances in neural information processing systems 2 (2002)
  849--856.

\bibitem{zhang2017robust}
Z.~Zhang, F.~Li, L.~Jia, J.~Qin, L.~Zhang, S.~Yan, Robust adaptive embedded
  label propagation with weight learning for inductive classification, IEEE
  transactions on neural networks and learning systems 29~(8) (2017)
  3388--3403.

\bibitem{zhu2019spectral}
X.~Zhu, Y.~Zhu, W.~Zheng, Spectral rotation for deep one-step clustering,
  Pattern Recognition (2019) 107175.

\bibitem{zhu2003semi}
X.~Zhu, Z.~Ghahramani, J.~D. Lafferty, Semi-supervised learning using gaussian
  fields and harmonic functions, in: Proceedings of the 20th International
  conference on Machine learning (ICML-03), 2003, pp. 912--919.

\bibitem{wang2009clustering}
F.~Wang, C.~Zhang, T.~Li, Clustering with local and global regularization, IEEE
  Transactions on Knowledge and Data Engineering 21~(12) (2009) 1665--1678.

\bibitem{nie2017multi}
F.~Nie, G.~Cai, X.~Li, Multi-view clustering and semi-supervised classification
  with adaptive neighbours., in: AAAI, 2017, pp. 2408--2414.

\bibitem{kang2017twin}
Z.~Kang, C.~Peng, Q.~Cheng, Twin learning for similarity and clustering: A
  unified kernel approach., in: AAAI, 2017, pp. 2080--2086.

\bibitem{zhuang2012non}
L.~Zhuang, H.~Gao, Z.~Lin, Y.~Ma, X.~Zhang, N.~Yu, Non-negative low rank and
  sparse graph for semi-supervised learning, in: Computer Vision and Pattern
  Recognition (CVPR), 2012 IEEE Conference on, IEEE, 2012, pp. 2328--2335.

\bibitem{de2013influence}
C.~A.~R. de~Sousa, S.~O. Rezende, G.~E. Batista, Influence of graph
  construction on semi-supervised learning, in: Joint European Conference on
  Machine Learning and Knowledge Discovery in Databases, Springer, 2013, pp.
  160--175.

\bibitem{peng2020deep}
X.~Peng, J.~Feng, J.~T. Zhou, Y.~Lei, S.~Yan, Deep subspace clustering, IEEE
  Transactions on Neural Networks and Learning Systems.

\bibitem{liu2013robust}
G.~Liu, Z.~Lin, S.~Yan, J.~Sun, Y.~Yu, Y.~Ma, Robust recovery of subspace
  structures by low-rank representation, IEEE Transactions on Pattern Analysis
  and Machine Intelligence 35~(1) (2013) 171--184.

\bibitem{maier2009influence}
M.~Maier, U.~V. Luxburg, M.~Hein, Influence of graph construction on
  graph-based clustering measures, in: Advances in neural information
  processing systems, 2009, pp. 1025--1032.

\bibitem{kang2017clustering}
Z.~Kang, C.~Peng, Q.~Cheng, Clustering with adaptive manifold structure
  learning, in: Data Engineering (ICDE), 2017 IEEE 33rd International
  Conference on, IEEE, 2017, pp. 79--82.

\bibitem{kang2019partition}
Z.~Kang, X.~Zhao, Shi, C.~Peng, H.~Zhu, J.~T. Zhou, X.~Peng, W.~Chen, Z.~Xu,
  Partition level multiview subspace clustering, Neural Networks 122 (2020)
  279--288.

\bibitem{zhang2010graph}
L.~Zhang, L.~Qiao, S.~Chen, Graph-optimized locality preserving projections,
  Pattern Recognition 43~(6) (2010) 1993--2002.

\bibitem{hou2009stable}
C.~Hou, C.~Zhang, Y.~Wu, Y.~Jiao, Stable local dimensionality reduction
  approaches, Pattern Recognition 42~(9) (2009) 2054--2066.

\bibitem{zhu2017robust}
X.~Zhu, X.~Li, S.~Zhang, C.~Ju, X.~Wu, Robust joint graph sparse coding for
  unsupervised spectral feature selection, IEEE transactions on neural networks
  and learning systems 28~(6) (2017) 1263--1275.

\bibitem{zhou2004learning}
D.~Zhou, O.~Bousquet, T.~N. Lal, J.~Weston, B.~Sch{\"o}lkopf, Learning with
  local and global consistency, in: Advances in neural information processing
  systems, 2004, pp. 321--328.

\bibitem{kang2015pca}
Z.~Kang, C.~Peng, Q.~Cheng, Robust pca via nonconvex rank approximation, in:
  2015 IEEE International Conference on Data Mining, IEEE, 2015, pp. 211--220.

\bibitem{kondor2008skew}
R.~Kondor, K.~M. Borgwardt, The skew spectrum of graphs, in: Proceedings of the
  25th international conference on Machine learning, 2008, pp. 496--503.

\bibitem{daitch2009fitting}
S.~I. Daitch, J.~A. Kelner, D.~A. Spielman, Fitting a graph to vector data, in:
  Proceedings of the 26th Annual International Conference on Machine Learning,
  2009, pp. 201--208.

\bibitem{eldridge2016graphons}
J.~Eldridge, M.~Belkin, Y.~Wang, Graphons, mergeons, and so on!, in: Advances
  in Neural Information Processing Systems, 2016, pp. 2307--2315.

\bibitem{kipf2016semi}
T.~N. Kipf, M.~Welling, Semi-supervised classification with graph convolutional
  networks, International Conference on Learning Representations (ICLR).

\bibitem{tang2019feature}
C.~Tang, X.~Liu, X.~Zhu, J.~Xiong, M.~Li, J.~Xia, X.~Wang, L.~Wang, Feature
  selective projection with low-rank embedding and dual laplacian
  regularization, IEEE Transactions on Knowledge and Data Engineering.

\bibitem{zhan2018multiview}
K.~Zhan, F.~Nie, J.~Wang, Y.~Yang, Multiview consensus graph clustering, IEEE
  Transactions on Image Processing 28~(3) (2018) 1261--1270.

\bibitem{zhang2020latent}
C.~{Zhang}, H.~{Fu}, Q.~{Hu}, X.~{Cao}, Y.~{Xie}, D.~{Tao}, D.~{Xu},
  Generalized latent multi-view subspace clustering, IEEE Transactions on
  Pattern Analysis and Machine Intelligence 42~(1) (2020) 86--99.

\bibitem{liu2018late}
X.~Liu, X.~Zhu, M.~Li, L.~Wang, C.~Tang, J.~Yin, D.~Shen, H.~Wang, W.~Gao, Late
  fusion incomplete multi-view clustering, IEEE transactions on pattern
  analysis and machine intelligence 41~(10) (2018) 2410--2423.

\bibitem{kang2020relation}
Z.~Kang, X.~Lu, J.~Liang, K.~Bai, Z.~Xu, Relation-guided representation
  learning, Neural Networks 131 (2020) 93--102.

\bibitem{bezdek2003convergence}
J.~C. Bezdek, R.~J. Hathaway, Convergence of alternating optimization, Neural,
  Parallel \& Scientific Computations 11~(4) (2003) 351--368.

\bibitem{liu2019multiple}
X.~Liu, X.~Zhu, M.~Li, L.~Wang, E.~Zhu, T.~Liu, M.~Kloft, D.~Shen, J.~Yin,
  W.~Gao, Multiple kernel k-means with incomplete kernels, IEEE transactions on
  pattern analysis and machine intelligence.

\bibitem{du2015robust}
L.~Du, P.~Zhou, L.~Shi, H.~Wang, M.~Fan, W.~Wang, Y.-D. Shen, Robust multiple
  kernel k-means using ℓ 2; 1-norm, in: Proceedings of the 24th International
  Conference on Artificial Intelligence, AAAI Press, 2015, pp. 3476--3482.

\bibitem{huang2015new}
J.~Huang, F.~Nie, H.~Huang, A new simplex sparse learning model to measure data
  similarity for clustering, in: Proceedings of the 24th International
  Conference on Artificial Intelligence, AAAI Press, 2015, pp. 3569--3575.

\bibitem{nie2014clustering}
F.~Nie, X.~Wang, H.~Huang, Clustering and projected clustering with adaptive
  neighbors, in: Proceedings of the 20th ACM SIGKDD international conference on
  Knowledge discovery and data mining, ACM, 2014, pp. 977--986.

\bibitem{huang2012multiple}
H.-C. Huang, Y.-Y. Chuang, C.-S. Chen, Multiple kernel fuzzy clustering, IEEE
  Transactions on Fuzzy Systems 20~(1) (2012) 120--134.

\bibitem{huang2012affinity}
H.~Huang, Y.~Chuang, C.~Chen, Affinity aggregation for spectral clustering, in:
  Computer Vision and Pattern Recognition (CVPR), 2012 IEEE Conference on,
  IEEE, 2012, pp. 773--780.

\bibitem{peng2018integrate}
C.~Peng, Z.~Kang, S.~Cai, Q.~Cheng, Integrate and conquer: Double-sided
  two-dimensional k-means via integrating of projection and manifold
  construction, ACM Transactions on Intelligent Systems and Technology (TIST)
  9~(5) (2018) 1--25.

\end{thebibliography}

\end{document}